\documentclass{article}
\usepackage{nips15submit_e}
\usepackage{graphicx}
\usepackage{times}
\usepackage{hyperref}
\usepackage{amsmath,amsfonts,dsfont}
\usepackage{graphicx} 
\usepackage{amssymb}
\usepackage{xcolor}
\usepackage{xspace}
\usepackage{algorithm,verbatim}
\usepackage{algorithmic}
\usepackage{tikz}
\usepackage{natbib}
\usepackage{url}
\usepackage{bbm}
\usepackage{amsthm}
\usepackage{wrapfig}

\renewcommand{\paragraph}[1]{{\noindent \bf #1}}

\newcommand{\proba}{\mathbb{P}}

\newcommand{\liva}[1]{\textcolor{blue}{#1}}

\usepackage{float}

\newtheorem{proposition}{Proposition}[section]
\newtheorem{corollary}{Corollary}[section]
\newtheorem{definition}[proposition]{Definition}
\newtheorem{lemma}[proposition]{Lemma}
\newtheorem{remark}[proposition]{Remark}
\newtheorem{assumption}{Assumption}
\newtheorem{theorem}{Theorem}
\newcommand{\Esp}{\mathbb{E}}

\newcommand{\C}{\mathcal{C}}

\renewcommand{\S}{\mathcal{S}}
\newcommand{\Sh}{\widehat{\mathcal{S}}}
\newcommand{\pareto}{\mathcal{P}}
\newcommand{\paretoh}{\widehat{\mathcal{P}}}
\newcommand{\width}{\text{\textbf{width}}}
\newcommand{\dal}{\mathcal{A}}

\renewcommand{\O}{\mathcal{O}}

\newcommand {\ep} {\varepsilon}

 \newcommand{\succe}{\succcurlyeq}

\newcommand{\unchain}{ {\tt Unchained Bandit}}

\newcommand{\algo}{ {\tt UBS Routine}\xspace}
\newcommand{\algochain}{ {\small\sf Sli\-cing\-Ban\-dits}\xspace}
\newcommand{\algopeeling}{ {\small\sf Un\-chai\-ned\-Ban\-dits}\xspace}
\newcommand{\pareps}{\parallel_\ep}
\newcommand{\reg}{\mathcal{R}}

\title{Decoy Bandits Dueling on a Poset}
\author{
Julien Audiffren \\
CMLA, ENS Cachan, CNRS\\
Paris Saclay University\\
94235 Cachan France \\
\texttt{audiffren@cmla.ens-cachan.fr} \\
\And
Liva Ralaivola \\
{\sc Qarma}, LIF, CNRS\\
Aix Marseille University \\
F-13289 Marseille cedex 9, France \\
\texttt{liva.ralaivola@lif.univ-mrs.fr}\\
}
\nipsfinalcopy

\begin{document}

\maketitle

\begin{abstract}
 We adress the problem of dueling bandits defined on partially ordered sets, or posets. 
 In this setting, arms may not be comparable, and there may be several (incomparable) optimal arms. 
We propose an algorithm, \algopeeling, that efficiently finds the set of optimal arms  of any poset
 even when pairs of comparable arms cannot be {\em distinguished} from pairs of incomparable arms, 
 with a set of minimal assumptions. 
 This algorithm relies on the concept of decoys, which stems from social psychology. 
 For the easier case where the incomparability information may be accessible, we propose a second algorithm,
  \algochain, which takes advantage of this information and achieves a very significant gain of performance
 compared to \algopeeling.  We provide theoretical guarantees and experimental evaluation for both algorithms.
 \end{abstract}


\section{Introduction}

\paragraph{Chasing the optimal set for cold-start recommendation.}
Today's recommendation systems heavily rely on machine learning. Dedicated
techniques may indeed be designed to extract statistical regularities from a
set of collected behaviors and provide users with relevant recommendations. 
%
One of the main challenges a recommendation system has to deal with is \textit{cold-start}, i.e. the 
situation where recommendations must be computed for a user for whom no information has been
collected.
A common strategy to get around this problem is to have at hand a set of default items to
recommend to any new customer. The design of such a set is then paramount to 
the user experience with the recommendation system and to his willingness to rely
on it for future movie suggestions. A natural goal is therefore to try to form
 a set of {\it best} movies.
Identifying the best movies is a task that requires a proper handling of two
features: a) the variety of existing film genres (documentary, drama, comedy\ldots) and b)  the uncertainty with which one film may be considered better than another.  The variety
of genres induces the issue of incomparability: there are pairs of movies ---comparison
of pairs is evidently at the core of the best movie selection
process--- that {\em cannot} be compared such as, e.g., a documentary
and a horror movie. This means
that movies are only {\em partially ordered} and it suggests
that the {\em set} of best movies {\em must} contain incomparable movies. Said
 otherwise, each movie from the set is the best in its category.
The uncertainty issue mentioned above then arises within a 
single genre as it might be complex to assert that a film is better than
another. A way to bypass this difficulty is to rely on a committee of critics
and to aggregate the (noisy) opinions of its members on pairs of comparable
movies. This might be implemented as follows:
for each pair of films, committee members are chosen at random and asked which of
the two movies is the better, and the movie that wins the most among the random
probes is decided to be the best. 
This introduction provides a practical motivation for the present paper
where we study the question of deriving strategies for {\em dueling bandits}
defined on {\em partially ordered sets}, or {\em posets}. We are in particular
interested in being able to find the set of best arms among all the (possibly incomparable)
arms at hand.

 \paragraph{Dueling Bandits on Posets.} 
 Dueling bandits were introduced by  \citet{yue2012if}. The setting, pertaining to the $K$-armed bandit framework,
 assumes there is no direct access to the reward provided by any single arm and the only information that can
 be gained is through the simultaneous pull of two arms: when such a pull is performed the agent gets
 access to the winner of the two arms,
 thus the name
 of {\em dueling} bandits. Here, we extend the framework of dueling bandits to the situation where 
 there exist pairs of arms that are not comparable, that is, we study the case where there might be no 
 natural order that could help decide the winner of a duel between two arms. A problem induced
 by such a framework is then to identify among the set of all available $K$ arms the set of {\em maximal} 
 incomparable arms,  or the {\em Pareto front}, through a minimal number of pairwise pulls.
 To carry out our study,
 we propose to make use of tools from the theory of posets and we take
 inspiration from works dedicated to selection and sorting on posets \cite{daskalakis2011sorting}.


\paragraph{Keys: Indistinguishability and Decoys.} We make the assumption that the underlying poset or, more precisely, 
 the incomparability structure, is not known. A pivotal issue that we have to face in this case is that of {\em indistinguishability}.
In the bandit setting we assume, the draw of two arms that are comparable and that have 
close values ---and hence a probability for either arm to win a duel close to 0.5--- 
is essentially driven by the same random process, i.e. an unbiased coin flip,
 as the draw of two arms that are not comparable.  Hence, if we denote by $\ep$ the distances between those two processes, we can have $\ep$ arbitrary small, and thus this pairs of arms cannot be distinguished from  an incomparable pair of arm on the sole basis 
of pulls and a well-thought strategy. Such pair of arm will be referred as $\ep$-indistinguishable. 
This problem has led us to make use of {\em decoy} arms. The idea of decoy originates from social
 psychology, and was originally intended to 'force' an agent (e.g., a customer) towards a specific
  good/action (e.g. a product) by presenting her a choice between the targetted good
  and a degraded version of it. Here, we use decoys to help solve the
 problem of indistinguishability 

\paragraph{Contributions.}  Our main contribution,  the \algopeeling algorithm, implements a 
strategy based on decoys and a peeling approach that finds the Pareto front\footnote{
We  discuss in the supplementary material
the easier setting where the incomparability information is known and we provide a dedicated algorithm, \algochain, that takes advantage of this addiditional information.} of a poset
 $\S$ with probability at least 1$- \delta$ after at most 
$T \le \O\left( K \frac{\width(\S)}{\Delta^2} {\log(NK^2/\delta) }\right)$ pairwise pulls, while incurring a regret upper bounded by 
$$\reg \le \frac{2K}{\gamma^2} \log\left( \frac{2N K^2}{\delta}\right)\sum_{i=1}^K \frac{1}{\Delta_i}   C_{\alpha,\gamma}(N_i) + K  \width(S) \log\left( \frac{2N K^2}{\delta}\right)\sum_{i, \Delta_i < \ep_{N-1}, i \notin \pareto}  \frac{1}{\Delta_i},$$
where $\Delta$ is the parameter of the decoys, $\Delta_i$ the regret associated with arm $i$, $K$ is the size of the poset, $\width(\S)$ its {\em width}, $N$ is the number of peeling iterations, $\gamma$ is the peeling rate, $\ep_{N-1}$ is the maximum peeling and $C_{\alpha,\gamma}(N_i) \le 1 $ defined in Section \ref{sec:Contribution} encodes the complexity of the poset with respect to arm $i$.


The paper is organized as follows. Section \ref{sec:background} presents the setting of dueling bandits on posets
and formally states the problem we address. 
 In Section \ref{sec:Contribution}, we formally introduce the notion of decoys and show how they can be
 constructed, both mathematically and practically, we then present our algorithm, \algopeeling, which relies on decoys,
  to find
the {\em exact} Pareto front of the poset and we provide theoretical
guarantees on their performances. In Section \ref{sec:relatedworks}, 
we discuss how the present work relates to recent papers from the dueling bandits
 literature. Section \ref{sec:experiments} reports results on the empirical performances of our 
algorithm in different settings.

\section{Problem Statement: Dueling Bandits on Posets}
\label{sec:background}

\subsection{Reminders on Posets}

We here recall base notions and properties about posets that are relevant
to the present contribution.

\begin{definition}[Poset]
Let $\S$ be a set of elements. {$(\S,\succcurlyeq)$} is a {\em partially ordered} or {\em poset} if $\succcurlyeq$ is a reflexive, antisymmetric and transitive binary relation on $\S$: $\forall a,b,c\in\S$
\begin{itemize}
	\item $a\succcurlyeq a$ (reflexivity);
	\item if $a\succcurlyeq b$ and $b\succcurlyeq a$ then $a=b$ (antisymmetry);
	\item if $a\succcurlyeq b$ and $b\succcurlyeq c$ then $a\succcurlyeq c$ (transitivity).
\end{itemize}
\end{definition}

\begin{remark}
	In the following, we will use $\S$ to denote indifferently the set $\S$ or the poset $(\S,\succcurlyeq)$, the distinction being clear from the context.
	We make use of the additional notation: $\forall a,b\in\S$
	\begin{itemize}
		\item $a\parallel b$ if $a$ and $b$ are {\em incomparable} (i.e. neither $a\succe b$ nor $b\succe a$);
		\item $a\succ b$ if $a\succcurlyeq b$ and $a\neq b$;
	\end{itemize}
	Throughout, we limit our study to {\em finite} posets, i.e., posets such that $|\S|<+\infty.$
\end{remark}


\begin{definition}[Maximal element and Pareto front]
 An element $a\in \S$ is said to be a {\em maximal element} of $\S$ if  $\forall b \in \S,$ $a \succcurlyeq b$ or $a \parallel b$. We denote by
$$\pareto(\S)\doteq\{a:a\succcurlyeq b\text{ or } a\parallel b, \forall b\in\S\},$$
 the set of maximal elements or Pareto front of the poset.
\end{definition}

Since there is no intrinsic reason to favor a particular maximal element, throughout this work we chose to focus on the task of finding the entire Pareto front $\pareto(\S)$ or $\pareto$, for short.  To this end, the notions of chain and antichain are  key.

\begin{definition}[Chain and antichain]
 $\C \subset \S$ is a {\em chain} (resp. an {\em antichain}) if $\;\forall a,b\in\C,\; a \succe b $ or $a \succe b $ (resp. $a\parallel b$).  $\C$ is said to be {\em maximal} if $\forall a\in\S\setminus\C,\; \C \vee a$ is not a chain (resp. an antichain).
\end{definition}

Note that $\pareto$ is by definition a maximal antichain. Finally, the notion of \textit{width} and \textit{height} of a poset are important to characterize (the \emph{complexity} of) a poset.
\begin{definition}[Width and height]
	The width (resp. height) of a poset $\S$ is the size of its longest antichain (resp. chain).
\end{definition}

\subsection{Dueling Bandits on Posets}

\paragraph{K-armed Dueling Bandit.} The $K$-armed dueling bandit problem \citep{yue2012if} assumes the existence of $K^2$ parameters $\{\gamma_{ij}\}_{1\leq i,j\leq K}$, with $\gamma_{ij}\in(-1/2,1/2)$ and the  following sampling procedure. 
	At each time step, the agent pulls a pair of arms $(i,j)$ and she gets
	in return the value of an independent realization of $B_{ij}$, a Bernouilli random
	variable with expectation $\Esp(B_{ij})=1/2+\gamma_{ij}$
	where $B_{ij}=1$ means that $i$ is the winner of the duel between $i$ and $j$ and, conversely, $B_{ij}=0$
	 means that $j$ is the winner. The objective of the agent is to find the {\em Condorcet winner} $c$---the arm such that $\gamma_{cj}>0,\;\forall j\neq c$--- among the $K$ arms, whose existence is assumed, 
	 while minimizing the accumulated regret, defined 
	  for
	 a sequence  $((i_1,j_1),\ldots,(i_T,j_T))$ of $T$ pairs of pulls by $\frac{1}{2}\sum_{t=1}^T(\gamma_{ci_t}+\gamma_{cj_t})$.
\begin{remark}
Note that: $\forall i,j,\; B_{ji}=1-B_{ij}$ and, thus, $\gamma_{ji}=-\gamma_{ij}$ and $\gamma_{ii}=0$.
\end{remark}

The implicit assumption of traditional dueling bandits is that the set $\S=\{1,\ldots,K\}$ of arms
is totally ordered: for any pair $i,j\in\S$ of arms, $i$ and $j$ must be comparable and $\gamma_{ij}$
{\em unequivocally} says which of the two is better.

\paragraph{Issues induced by working on posets.} Now consider a dueling bandit problem defined on a poset 
$\S$. 
Compared to the usual setting where a total order on the arms exists,
 there are two main differences which arise 
when $\S$ is a poset: first, the situation where the agent pulls a pair of arms that are not comparable has to be handled with care and, second, there might be multiple maximal elements.

Working on bandits with a poset $\S=\{1,\ldots,K\}$ of arms might be formalized as follows.
For all chains $\{i_1,\ldots,i_m\}$ of $m$ arms there exist a family  $\{\gamma_{i_pi_q}\}_{1\leq p,q\leq m}$
of parameters such that $\gamma_{ij}\in(-1/2,1/2)$; the pull of a pair of arms $(i_p,i_q)$ from the
same chain provides the realization of a Bernoulli random variable $B_{i_pi_q}$ with expectation 
$\Esp(B_{i_pi_q})=1/2+\gamma_{i_pi_q}.$ Regarding the incomparability, i.e. the situation where
the pair of arms $(i_p,i_q)$ selected by the agent correspond to arms such that $i_p\parallel i_q$,
then there are two frameworks we propose to consider: one the one hand, the {\em fully observable
posets}, where the draw from an incomparable pair of arms provides the agent with the information regarding
the comparability of the arms\footnote{This easier setting is analysed in depth in the supplementary materials.}.  On the other hand, that of {\em partially observable posets}, where
such a draw is modeled as the toss of an unbiased coin flip---as we shall discuss, this framework poses
 the problem of indistinguishability mentioned in the introduction.
 
\paragraph{Regret on posets.}  In order to extend the notion of regret associated to an arm $i$, $\Delta_i,$ in the poset setting, we use the notion of distance to the Pareto front, noted $d(i,\pareto)$ defined as follows :
$$\Delta_i = d(i,\pareto)=\min \{\gamma_{ij}, \forall j \in \pareto \text{ such that } j \succe i \}.$$ We then define the regret occurred by comparing two arms $i$ and $j$ by $d(i,\pareto) + d(j,\pareto)$.  It is important to remark that the regret of a comparison is zero if and only if the agent is comparing two element of the Pareto front.

\paragraph{Problem statement.}
Given the issues induced by working on a poset $\S$ of arms, we may state that the problem that we want to
tackle is to identify the Pareto front $\pareto(\S)$ of $\S$ as efficiently as possible.
%
More precisely, we want to devise pulling strategies for both poset observability frameworks such that
for any given $\delta\in(0,1)$, we are ensured that the agent is capable, with probability $1-\delta$ 
to identify $\pareto(\S)$ with controlled number of pulls {\em and} regret.


\begin{assumption}[Order Compatibility] 
$$\forall i,j  \in \S, \quad (i \succ j) \text{ if and only if } \gamma_{ij}>0.$$
\end{assumption}

We will not require any further hypothesis on how the $\gamma_{ij}$ relate to each other
and, therefore, no assumption on \emph{strong stochastic transitivity} \citep{yue2012if} is required.

\subsection{Poset Observability}
We consider the following setting, where the uncomparability information is not accessible.

\paragraph{Partially observable posets.}
	A $K$-armed Dueling bandits on a partially observable poset $\S=\{1,\ldots,K\}$ is a dueling bandit problem such that if $i \parallel j$, then $\gamma_{ij}=0.$ This property is referred as {\em Partial Observability}.






This property reflects the fact that
 neither of the two incomparable arms has a distinct advantage over the other: when the agent asks to compare 
 two intrinsically incomparable arms, the results
  will only depend upon circumstances independent from the arms (like luck or personal tastes).
  Our encoding of such framework makes us assume that when considered over many pulls, the effects
 of those circumstances cancel out, so that no specific arm is favored, whence $\gamma_{ij}=0$.

\paragraph{Consequences of partial observability.} Note that partial observability entails the problem
of indistinguishability evoked previously. Indeed, given two arms $i$ and $j$, regardless of the number
 of comparisons, an agent may never be sure if either the two arms are very close to each other ($\gamma_{ij} \approx 0$  and i and j are comparable) 
 or if they are not comparable ($\gamma_{ij} =  0$).
  Since all the elements of the Pareto set must be incomparable with each other, 
 this renders the problem of identifying $\pareto$ intractable as well if no additional information is provided.
 
This problem motivates the following definition, which quantifies the notion of indistinguishability :
\begin{definition}[$\ep-$indistinguishability] 
Let $a,b \in \S$ and $\ep>0$. $a$ and $b$ are said to be $\ep$-indistinguishable, noted $a \pareps b$, if $\vert \gamma_{ab} \vert \le \ep.$
\end{definition}


As the notation $\pareps$ implies, the $\ep-$indistinguishability of two arms can be seen as a weaker form of incomparability, and note that
 as $\ep$-decreases, previously indistinguishable pairs of arms become distinguishable, and the only $0-$indistinguishable pair of arms are the incomparable pairs. The classical notions of a poset related to incomparability can easily be extended  to fit the  $\ep-$indistinguishability :

\begin{definition}[$\ep$-antichain, $\ep$-width and $\ep$-approximation of $\pareto$]\label{def: ep approx to pareto}
Let $\ep>0$.
 $\C \subset \S$ is called an $\ep$-antichain if $\forall a \neq b \in \C$, we have $a\pareps b$.
Additionally, $\pareto' \subset \S$ is called an $\ep-$approximation of $\pareto$ if 
$\pareto \subset \pareto'$ and $\pareto'$ is an $\ep$-antichain. Finally we denote by $\width_\ep(\S)$ the size of the largest $\ep-$ antichain of $\S$.
\end{definition}

Interestingly, to find a $\ep-$approximation of $\pareto$, it is only needed to remove the elements of $\S$ which are $\ep-$distinguishable from $\pareto$. Thus, while $\pareto$ cannot be recovered in the partially observable setting, a $\ep-$approximation of $\pareto$ can be obtained.
Consequently, if the agent knows the minimum distance of any arm to the Pareto set, 
defined as $d(\pareto)=\min \{\gamma_{ij}, \forall i \in \pareto, j \in \S\setminus \pareto, \text{ such that } i \succ j \},$ she can recover the Pareto front, since for any $\ep< d(\pareto)$, the unique $\ep-$approximation of $\pareto$ is $\pareto$ itself. 

This information is however unavailable in practice and we choose
not to rely on external information to solve the problem at hand. In the case where an $\ep$ approximation of the Pareto front is not enough, and the \emph{exact} Pareto front is required,  we devise an alternative strategy which rests on the idea of {\em decoys},
already mentioned in the introduction and fully developed in Section~\ref{sec:Contribution}.

\section{Contributions}
\label{sec:Contribution}
Here, we introduce our algorithm, \algopeeling,
that solves the problem of dueling bandits on partially observable posets, 
and we provide theoretical performance guarantees. 

\subsection{Decoys and Posets}\label{subsec: decoy}

\begin{algorithm}[t!] 
\caption{Direct comparison}
\begin{algorithmic}[2bis] \label{algodirect}
\STATE \underline{\textbf{Given}} $(\S,\succ)$ a poset, $\delta, \ep >0$, $a,b \in \S$

\STATE \underline{\textbf{Initialisation}} Maintains $p_{ab}$ the average number of victory of a over b and  $I_{ab}$ its $1-\delta$ confidence interval,
\STATE \underline{\textbf{Direct comparison:}}
\WHILE{ $ 0.5+\ep \in I$ or $ 0.5-\ep \in I$ }
\STATE Compare $a$ and $b$, Update $p_{ab}$ and  $I$.
\STATE \textbf{If} $0.5 \notin I_{ab}$ and $p_{ab}>0.5$, \textbf{Return} $a\succ b$; \textbf{Else} \textbf{If} $0.5 \notin I_{ab}$ and $p_{ab}<0.5$, \textbf{Return} $b \succ a$.
\ENDWHILE
\STATE \textbf{Return} $a \parallel_\ep b$
\end{algorithmic}
\end{algorithm}

As said in Section~\ref{sec:background}, deciding if two arms are incomparable or very close is intractable
in the partially observable poset, and so is that of finding the {\em exact} Pareto front.

Still, without any additional device, the agent is able to find if  two arms $a$ and $b$, are $\ep$-indistinguishable. 
using the {\em direct comparison} process provided by Algorithm~\ref{algodirect}. 
Yet, as previously discussed, this  only produces an $\ep$-approximation of the Pareto front, of whom  $\pareto$ is only guaranteed to be a \emph{subset}.
To evade this shortcoming, we introduce a new tool, {\em decoys}, inspired by works from social
 psychology \citep{huber1982adding}. We formally define decoys for posets, 
 and we prove that it is a sufficient tool to solve the incomparability problem (Algorithm \ref{algodecoy}).
 We also present methods for building those decoys, both for the purely formal model of
  posets and for real-life problems.

\begin{definition}[$\Delta$-decoy]\label{definition decoy}
Let $a \in \S$. Then $b\in\S$ is said to be a $\Delta$-decoy of $a$ if :
\begin{enumerate}
\item $a \succe b$ and $\gamma_{a,b}\ge \Delta$
\item $\forall c \in \S, a \parallel c$ implies $b \parallel c$
\item $\forall c \in \S$ such that $c \succe a$, $\gamma_{c,b} \ge \Delta$
\end{enumerate}
\end{definition}

Interestingly, when $\S$ satisfies the strong  stochastic transitivity hypothesis, 
the third point of the previous definition in an immediate consequence of the first.
The following proposition illustrates how decoys can be used to determine the incomparability of two arms.
\begin{proposition}[Decoys and incomparability]\label{prop: decoy comparable}
Let $a$ and $b \in \S$. Let $a'$ (resp. $b'$) be a $\Delta$-decoy of $a$ (resp. $b$). Then $a$ and $b$ are comparable if and only if $\max(\gamma_{b,a'},\gamma_{a,b'})\ge \Delta$.
\end{proposition}

\begin{proof}
Let us assume than $a\succe b$. The transitivity of $\succe$ implies that $a \succe b'$, and the third point 
of Definition~\ref{definition decoy} implies that $\gamma_{a,b'}\ge \Delta$. 
The rest follows from point 2 of Definition~\ref{definition decoy}.
\end{proof}

\begin{algorithm}[t!] 
\caption{Decoy comparison}
\begin{algorithmic}[2bis] \label{algodecoy}
\STATE \underline{\textbf{Given}} $(\S,\succ)$ a poset, $\delta, \ep >0$, $a,b \in \S$

\STATE \underline{\textbf{Initialisation}} Create $a', b'$ the respective $\ep$- decoy of $a, b$. Maintains $p_{ab}$ the average number of victory of a over b and  $I_{ab}$ its $1-\delta/2$ confidence interval,
\STATE \underline{\textbf{Decoy comparisons:}}
\WHILE{  $ 0.5+\ep \in I$ }
\STATE Compare $a$ and $b'$, $b$ and $a'$, Update $p$, and $I$.
\STATE \textbf{If} $0.5 \notin I_{ab'}$ and $p_{ab'}>0.5$, \textbf{Return} $a\succ b$.
\textbf{Else} \textbf{If} $0.5 \notin I_{ba'}$ and $p_{ba'}>0.5$, \textbf{Return} $b \succ a$.
\ENDWHILE
\STATE \textbf{Return} $a \parallel b$
\end{algorithmic}
\end{algorithm}

Algorithm~\ref{algodecoy} is derived from this result. The next proposition, an immediate consequence of Proposition~\ref{prop: decoy comparable}, gives a theoretical guarantee on its performances.

\begin{proposition}\label{prop:use of decoy}
Algorithm \ref{algodecoy} returns the correct incomparability result with probability at least $1-\delta$
 after at most n comparisons, where $n=4 {\log(4/\delta)}/{\Delta^2}.$
\end{proposition}

\paragraph{Adding decoys to a poset.} A poset $\S$ may not  contain all the necessary decoys. To alleviate this, the following 
proposition states that it is always possible to add relevant decoys to a poset.
\begin{proposition}[Extending a poset with a decoy.]\label{prop: add decoy}
Let $(\S,\succe,\gamma)$ be a dueling bandit problem on a partially observable poset,
 and $a \in \S$. Define $a',S',\succ',\gamma'$ as follows:
\begin{itemize}
\item $\S' = \S \cup \{ a'\}$
\item $\forall b,c \in \S,$ $b\succe c$ i.f.f. $b \succe' c,$,  and $\gamma'_{b,c}=\gamma_{b,c}$
\item $\forall b \in \S$, if $b \succe a $ then $ b \succe a'$ and $\gamma'_{b,a'}=\max(\gamma_{b,a},\Delta)$. Otherwise, $b \parallel a'$.
\end{itemize}
Then $(S',\succe')$ is a poset and $(\S',\succe',\gamma')$ defines a dueling bandit problem on a partially observable poset, $\gamma'_{\vert S}=\gamma$, and $a'$ is a $\Delta$-decoy of $a$.
\end{proposition}
\begin{proof}
The result naturally follows from the definition of a poset and Definition~\ref{definition decoy}.
\end{proof}

\paragraph{Decoys in real-life problems.} 
%
%
The intended goal of a decoy $a'$ of $a$ is to have at hand an arm that is known to be lesser than $a$.  
Creating such a decoy in real-life can be done by using a degraded version of $a$: for the case of a movie,
a decoy can be obtain by e.g. decreasing the resolution of a film.
Note that while for large values of the $\Delta$ parameter of the decoys Algorithm \ref{algodecoy} requires less comparisons (see Proposition \ref{prop:use of decoy}), in real-life problems, the second point of Definition \ref{definition decoy} tends to
  becomes false: the new option is actually so worse than the original that the decoy becomes comparable (and inferior) to \emph{all} the
   other arms, including previously non comparable arms (example: the decoy of a film for a very large $\Delta >0$ could be in very low resolution such as $32\times 24$; this film cannot be actually seen and is clearly worse than all the others, regardless of the genre). In that case, the use of decoys of arbitrarily large $\Delta$ can lead to erroneous conclusions about the Pareto front and should be avoided.

\subsection{\algopeeling}\label{subsec : unchained}
We now present our algorithm, \algopeeling, that uses decoys to efficiently find the Pareto front of $\S.$
\algopeeling is inspired by the ideas developed by \citet{daskalakis2011sorting}, who address the problem of \emph{sorting} a poset in a noiseless environment.

By Proposition \ref{prop:use of decoy}, Algorithm \ref{algodecoy} can be used to  establish the exact relation between two arms. But this process can be very costly, as the number of required comparison is proportional to $1/{\Delta^2}$, even for strongly suboptimal arms. 
To avoid this possibility, \algopeeling implements a peeling technique: given $N>0$ and a decreasing sequence $\left(\ep_i\right)_{i=1}^{N-1}$ it computes and refines an $\ep_i$-approximation of the Pareto front $\paretoh_i$, using a subroutine (Algorithm \ref{algounchained}), which considers $\ep_i$-indistinguishable arms as incomparable.
Then, at the $N$-th epoch, it uses Algorithm \ref{algounchained} one final time where it uses Algorithm \ref{algodecoy} with $\Delta$-decoys for comparisons, and then returns the Pareto front.

\begin{algorithm}[t!]
\caption{\algopeeling	}
\begin{algorithmic}  \label{algopeeling}
\STATE \underline{\textbf{Given}} $\S=\{s_1,\ldots,s_K\}$ a poset,  $\delta>0$, $\Delta>0$, $N>0$, $(\ep_i)_{i=1}^{N-1} \in \mathbb{R}_+^{N}$

\STATE \underline{\textbf{Initialisation}} Set $\S_1=\S$, $\ep_N=\Delta$.
\STATE \underline{\textbf{Peel $\paretoh$}} \textbf{for} $t=1$ \textbf{to} $N-1$ \textbf{do}
 \small{$\S_{t+1}$ = \algo$\left( \S_t, \ep_t, \delta/N, \dal=\text{ Algorithm 3} \right)$.}
\textbf{ end for}
\STATE \underline{\textbf{Use decoys}} \small{$\paretoh$ = \algo$\left( \S_{N}, \Delta, \delta/N, \dal=\text{ Algorithm 2} \right)$.}\\
\textbf{RETURN} $ \paretoh  $
\end{algorithmic}
\end{algorithm}

\begin{algorithm}[t!]
\caption{\algo}
\begin{algorithmic} \label{algounchained}
\STATE \underline{\textbf{Given}} $\S_t$ a poset,  $\ep_t > 0$ a precision criterion, $\delta'$ an error parameter, $\dal$ a comparison algorithm

\STATE \underline{\textbf{Initialisation}} Choose $p\in \S_t$ at random.  Define $\paretoh =\{ p\}$ the set of pivots.

\STATE \underline{\textbf{Construct $\paretoh$}}
\FOR { $c \in \S_t \setminus\{p\}$}
	\FOR { $c' \in \paretoh $}
		\STATE Compare $c$ and $c'$ using $\dal(\delta=\delta'/\vert S_t \vert^2,\ep=\ep_t).$ \textbf{If} $c \succ c'$, \textbf{Then} remove  $c'$ from $\paretoh$.  
	\ENDFOR
	\STATE \textbf{If} $\forall c' \in \paretoh$, $c\parallel c'$, \textbf{Then} add $c$ to $\paretoh$
\ENDFOR\\
\textbf{Return} $ \hat{\pareto}  $
\end{algorithmic}
\end{algorithm}

\paragraph{Algorithm subroutine.} Algorithm \ref{algounchained} called on $\Sh$ with parameter $\ep>0$, $\delta >0$ and $\dal$ works 
as follows. It chooses a single initial {\em pivot}---an arm to which other arms are compared---and successively  examines 
all the elements of $\Sh$. Each of the examined element $p$  is compared to all the pivots. 
Each pivot that is dominated by $p$  is removed from the pivot set. Then if after being compared to all the pivots, $p$ was dominated by none,  it is added to the pivot set. At the end, the set of remaining pivot is returned. 
During the first $N-1$ epochs, the comparisons are done with Algorithm \ref{algodirect}. In the last epoch, the agent uses Algorithm~\ref{algodecoy} to get exact information on the relations between the remaining arms. 


\paragraph{Reuse of informations.} To optimize the efficiency of the peeling process, \algopeeling 
reuses previous comparison results. At the beginning of each \emph{direct} comparison process between arms $a$ and $b$,
 the empirical estimate $p_{ab}$ and its confidence interval $I_{ab}$  are initialized using the results of the previous direct
  comparisons of $a$ and $b$. However, no information can be reused in the last epoch for the remaining arms, 
  as the indirect comparison algorithm does not compare $a$ to $b$ directly.

The following theorem gives a high probability bound on the performances of \algopeeling.

\begin{theorem}\label{thm:peeling regret}
The \algopeeling algorithm applied on $\S$ with parameters $\delta,\Delta$,$N$ and  with a decreasing sequence $(\ep_i)_{i=1}^{N-1}$ lower bounded by $\Delta\sqrt{\frac{K}{\width(\S)}}$,
returns the Pareto front $\pareto$ of $\S$ with probability at least $1- \delta$ after at most  $T$ comparisons, with
\begin{equation}\label{eq: peeling regret}
\begin{aligned}
T \le \O\left( K {\width(\S)} {\log(NK^2/\delta)}/{\Delta^2}\right)
\end{aligned}
\end{equation}
\end{theorem}
This is a consequence of the following intermediate result, whose proof can be found in the supplementary materials.
\begin{proposition}\label{prop:unchainedregretfinal}
Algorithm \ref{algounchained} called on $\S_t$ with parameter $\ep_t>0,$ $\delta'>0$ and $\dal$ = Algorithm \ref{algodecoy} returns the Pareto front of $\S_t$ with probability at least $1- \delta'$ after at most
$$T \le 4\vert \S_t \vert \text{\textbf{width}}(\S_t) {\log(4\vert \S_t \vert^2/\delta') }/{\Delta^2} $$
comparisons. Alternatively, when Algorithm \ref{algounchained} uses $\dal$ = Algorithm \ref{algodirect}, it returns an $\ep_t-$approximation of the Pareto front of $\S_t$ with probability at least $1- \delta'$ after at most
$$T \le 2 \vert\S_t \vert \width_{\ep_t}(\S_t) {\log(2 \vert \S_t \vert^2 /\delta') }\left(\frac{1}{\ep_t^2}- \mathbf{1}_{t>1}\frac{1}{\ep_{t-1}^2} \right) $$
\textbf{additional} comparisons, where $\mathbf{1}$ is the  indicator function.
\end{proposition}

\begin{proof}[Proof of Theorem \ref{thm:peeling regret}.]
Note that $\forall \S'\subset \S$ such that $\pareto \subset \S'$, $\pareto(\S')=\pareto$.
The result is obtained by summing the upper bound in Proposition \ref{prop:unchainedregretfinal} over the different epochs,
rearranging the sum and using the fact that $\vert \S_t \vert \width_{\ep_t}(\S_t) {\log(N\vert \S_t \vert^2 /\delta) }$ is decreasing in $t$ while $1/{\ep_t^2}$ is increasing in $t$. The detailed proof can be found in the supplementary materials.
\end{proof}

\paragraph{Peeling rate.} Note that even if $\width(\S)$ is unknown, it suffices to choose
\begin{equation} \label{eq: lower bound on ep}
\ep_{N-1} \ge \Delta \sqrt{K}
\end{equation}
to satisfy the hypotheses of Theorem~\ref{thm:peeling regret}. Although the previous result is valid for any decreasing sequence of $(\ep_t)$ satisfying \eqref{eq: lower bound on ep}, we focus on geometrically decreasing sequences, i.e. $\exists \gamma>0$ such that $\ep_t = \gamma^t \ep_0$. For the sake of simplicity, we set $\ep_0=0.5$ but all the following results can easily be extended for any $\ep_0 >0$.

Before we present our regret upper bound, we need to introduce a few notations. In order to characterise the efficiency of the peeling approach associated to $\gamma$, we define $\alpha(\gamma,\ep_0,N) \in ]0,1]$, or $\alpha$ for short,  as follows :
\begin{equation}\label{eq:def alpha}
\alpha = \inf\left\lbrace a \in [0,1], \text{ s.t. } \forall 1 \le t \le N, \forall S_t \subset \S \text{ an } \ep_0 \gamma^t \text{-approximation of } \pareto, \vert S_t \vert \le a^t \vert \S \vert   \right\rbrace.
\end{equation} 
It is important to note that the previous inequality is always true for $a=1$, so $\alpha$ is always defined. $\alpha$ characterise how efficient is the peeling for the chosen parameters by quantifying the reduction in size between the successive $\S_t$. We can now introduce the following Theorem which gives an upper bound on the regret incurred by \unchain.

\begin{theorem}\label{thm: peeling upper bound}

Let $\reg_0$ (resp. $\reg_1$) be the regret generated by Algorithm \ref{algopeeling} applied on $\S$ with parameters $\delta,\Delta$,$N$ and  with a decreasing sequence $(\ep_i)_{i=1}^{N-1}$  such that $ \ep_{N-1} \le \Delta\sqrt{K}$ during the peeling phase (resp. the decoy phase). Let $\alpha$ as defined by \eqref{eq:def alpha}.
 Then $\reg = \reg_0+ \reg_1$  and with probability at least $1-\delta$,
\begin{equation}\label{eq: upper bound reg_0}
\reg_0  \le \frac{2K}{\gamma^2} \log( \frac{2N K^2}{\delta})\sum_{i=1}^K \frac{1}{\Delta_i}   C_{\alpha,\gamma}(N_i)
\end{equation}
\begin{equation}\label{eq: upper bound reg_1}
\reg_1  \le K  \width(S) \log( \frac{2N K^2}{\delta})\sum_{i, \Delta_i < \ep_{N-1}, i \notin \pareto}  \frac{1}{\Delta_i},
\end{equation}
where
 \begin{equation*}
 N_i  = \min \left( \lceil \frac{\log(\Delta_i)}{\log(\gamma)} \rceil, N-1 \right)
  \quad  \text{and} \quad 
 C_{\alpha,\gamma}(n) = \left\lbrace \begin{aligned}
 &n \alpha^{n-1} \text{ if } \alpha= \gamma, \\
 & \frac{\gamma^{2n}(1-\alpha) + \alpha^{n}(\gamma^2 -1)  }{\gamma^2 - \alpha} \text{ otherwise. }
 \end{aligned}\right.
 \end{equation*}
\end{theorem}

In the previous theorem, $N_i$ represent the number of peeling step where the arm $i$ is present, while $C_{\alpha,\gamma}(N_i)$ represent the cost of doing the peeling for arm $i$. It is worth noting that $C_{\alpha,\gamma}(n) \le 1$ and is increasing in $ \alpha$. This reflect the fact that small $\alpha$ are representative of an efficient pruning (many arms removed at each step).

\paragraph{Opposite constraints on $\ep$.}
Theorem~\ref{thm:peeling regret} is an upper bound on the number of comparisons required to find the Pareto  front. This bound is tight in the 
(worst-case) scenario where all the arms are $\Delta$-indistinguishable, i.e. peeling cannot eliminate any arm. In that case, any comparison done during the peeling is actually wasted, and the lower bound on $\ep_t$  \eqref{eq: lower bound on ep} allows to upper bound the number of comparisons made during the peeling step to recover a $K\width(S)$  dependency in the upper bound, instead of $K^2$. On the other hand, a significant amount of peeling is required to obtain a reasonable upper bound on the incurred regret: the number of comparisons using decoys is very high $( \approx 1/ \Delta^2)$ and is the same for every arm, regardless of its regret. So it is important that only near-optimal arms remain during the decoy step, hence the upper bound on $\ep_t$. In order to satisfy both constraints, $\ep_N$ must be chosen in $\left[ \sqrt{K/\width(S)} \Delta, \sqrt{K} \Delta\right].$

\section{Related Works}
\label{sec:relatedworks}

There is an actual connection between
our work and studies from social psychology. In particular, \citet{tversky1981framing} issued one of
the reference papers on the choice problem---which pertains to comparisons, in our framework---
 for real-life problems; they introduced the idea that alternatives may influence the perceived value 
 of items. This idea had been taken one step further by \citet{huber1982adding}, who introduced 
 and formalized the idea of decoys. They specifically argued 
 that introducing {\em dominated alternatives}, i.e. decoys, may increase the probability
  of the original item to be selected: if $A,B$ and $A'$ are alternatives, then
$\proba\left( \text{select }A\text{ among } A,B\right)  < \proba\left( \text{select } A \text{ among } A,A',B\right)$.
This generated an abundant literature (see \citet{ariely1995seeking,sedikides1999contextual} and references therein)
on works that studied the effect of decoys and their uses in various fields.

From the computer science literature, we must mention the work of \citet{daskalakis2011sorting}, which
addresses the problem on selection and sorting on posets and provides relevant data structures and 
accompanying analyses for computing on posets. Their results come down to classical results when totally
ordered sets are used. Also, there might be yet other connections to draw between our work
and that of \citet{feige1994computing} who tackle the problem of sorting with noisy comparisons; note however
that they assume there is a total order on the items they work on and the connection to be made with the
present work would be to identify how this assumption may be weakened, if not removed.

Finally, we must discuss how our contribution separates from papers on dueling bandits.
If the seminal paper of \citet{yue2012if} promotes algorithms, namely the Interleaved Filter
 algorithms,
that exhibit optimal information-theoretic regret bounds, the authors assume the existence of 
a total order between the arms together as strong stochastic transitivity
and (relaxed) stochastic triangle inequality. Since then, numerous methods have been proposed 
 to relax those additional assumptions, including  \citep{yue2011btm,ailon2014reducing,zoghi2014rucb,zoghi2015mergerucb}.
Other approaches exist that do not assume the existence of a Condorcet winner, 
such as \citep{urvoy2013savage,busa2013top,zoghi2015copeland} but, to the best of our knowledge,
we provide the first contribution that studies the framework where arms may be {\em incomparable}.

\section{Numerical Simulations}
\label{sec:experiments}

\begin{figure*}[tb]\label{fig:results}
\begin{center}
\includegraphics[scale=0.18]{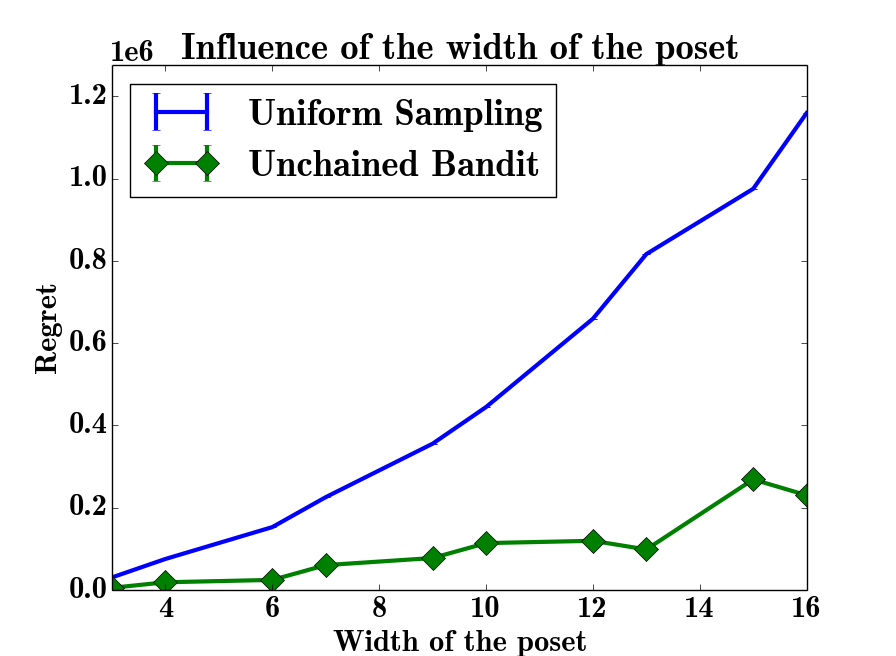}
\includegraphics[scale=0.18]{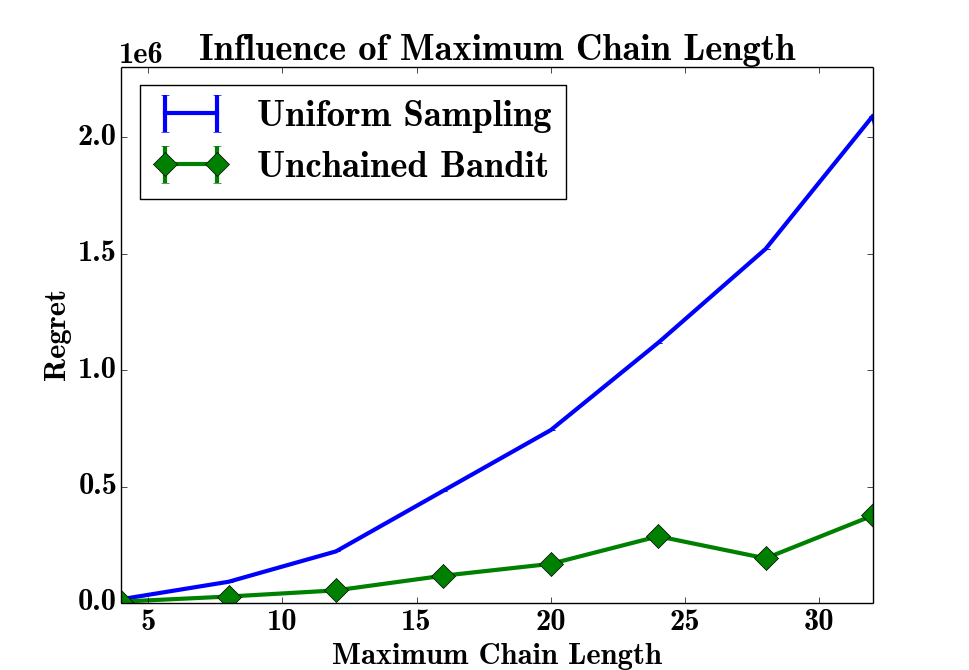} 
\includegraphics[scale=0.18]{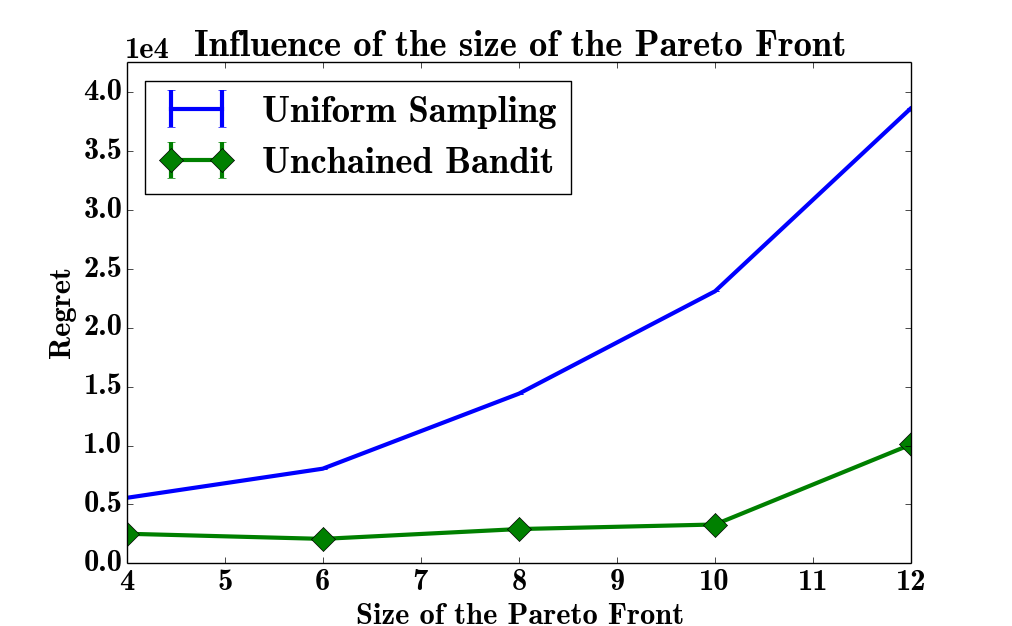}
\end{center}
\caption{Time necessary to reach a conclusion for \algochain and  \algopeeling compared to {\small \sf UniformSampling}, when the structure of the poset varies. Dependence on (left:) width of the poset, (center:) height of the poset and (right:) size of the Pareto front. \liva{Don't talk about \algochain.}}
\end{figure*}
In this section we experimentally evaluate \algopeeling. We did not compare our algorithm to dueling bandits algorithms from the literature,
 as a) they fail to consider the incomparability information 
  and b) they are generally designed to return only one \emph{best} element. Instead, we studied the performances of \algopeeling on simulated data (Section 5.1), and we applied it to an existing film rating database (Section 5.2).

\subsection{Simulated Poset}
First we confront \algopeeling with randomly generated posets,  with different sizes, widths and heights. In order to give a baseline value, we use a simple algorithm,   {\small \sf UniformSampling} inspired from the successive elimination algorithm \cite{even2006successiveelimination}, which
    simultaneously compares  all possible pairs of arms  until one of the arms appears suboptimal, at which point
	 it is removed from the set of selected arms. When only $\Delta$-indistinguishable elements remain, it uses $\Delta$-decoys. 

Given the size $p>0$ of the Pareto front, the desired width $w\ge p$ and the height $h>0$, the posets are generated as follows:
 first, a Pareto front of size $p$ is created. Then $w$ chains of length $h-1$ with no common elements are added. Finally, the top of the chains are connected to a random number of elements of the Pareto front. This creates the structure of the poset (i.e. the partial order $\succ$).  Finally, the exact values of the $\gamma_{ij}$'s are obtained from a uniform distribution, conditioned to satisfy the partially (or fully) observable framework.
When needed, $\Delta$-decoys are created according to Proposition~\ref{prop: add decoy}.
For each experiment reported on Figure~1, we changed the value of one parameter, and left the other to 
their default values ($p=5$, $w=2p$, $h=10$). The results are averaged over ten runs. By default, we use $\delta= 1/1000$ and
 $\Delta=1/100$. We also set $\gamma=0.9,$ $\ep_0=0.5$ and $\ep_N= \sqrt{K} \Delta.$ 
 
We note that for partially observable posets, \algopeeling produces much better results than {\small \sf UniformSampling} and its advantage increases with the complexity of the problem.

\subsection{MovieLens Dataset}

To illustrate the example of the films recommendation system developed in the introduction, we chose to apply \algopeeling to the 20 millions items MovieLens dataset (\cite{harper2015movielens}).

To simulate a dueling bandit on a poset  we proceed as follows: we remove all films with less than $50000$ evaluations, thus obtaining 159 films, represented as arms. Then, when comparing two arms, we picke at random a user which has evaluated {\em both} films, and compare those evaluations (ties are broken with an unbiased coin toss). Since the decoy tool cannot be used in an already existing dataset, we restrict ourselves to finding an $\ep$-approximation of the Pareto front, with $\ep=0.05$.
Then, \algopeeling is run  with parameters $\gamma=0.9$, $\ep_0=0.5$, $\ep_N=\ep$, $\delta=0.001$. 

There is no known ground truth for this experiments, so no regret estimation can be provided. Instead, the resulting Pareto front, which contains 5 films, is listed in Table \ref{tab: films names}, and compared to the five films among the original 159 with the highest average score.
It is interesting to note that three films are present in both list, which reflects the fact that the {\em best } films in term of average score have a high chance of being in the Pareto Front. On the other hand, the films contained in the Pareto front are more diverse in term of genre, which is expected of a Pareto front.
For instance, the sequel of the film "The Godfather" (hence very close to the original regarding  genre) has been replaced by a a film of a totally different genre.  It is important to remember that  \algopeeling \emph{does not have access to any information about the genre of a film} and its results are based solely on the pairwise evaluation of the user, and thus this result illustrates the effectiveness of our approach for the learning of the hidden poset.

\begin{figure}\label{tab: films names}
\centering
\begin{tabular}{c|c}
Pareto Front & Highest average score \\
\hline
\\
Pulp Fiction   & Pulp Fiction  \\
Fight Club & The Usual Suspect \\
The Shawshank Redemption & The Shawshank Redemption \\
The Godfather & The Godfather \\
Star Wars Episode V & The Godfather: Part II 
\end{tabular}
\caption{Comparison between the five films with the highest average score (right column) and the five films of the $\ep-$pareto set (left column)  }
\end{figure}


\section{Conclusion}
\label{sec:conclusion}
We studied an extension of the dueling bandit problem to the poset framework,
 which raised the problem of $\ep$-indistinguishability. We presented a new algorithm, \algopeeling , 
 which tackles the partially observable settings, and we provided theoretical performance guarantee
for its ability to identify the Pareto front. Future work might include the study of the influence of additional hypothesis on the structure of the poset, such as when the poset is actually a lattice or upper semi-lattice. In this case, different strategies of sampling might lead to even more efficient algorithms.

\newpage  
\bibliographystyle{plainnat}
\bibliography{unchainedbandit}
\newpage
\appendix

\section{Fully Observable Posets, \algochain}
\label{sec: algochain}

\begin{algorithm}[t!] 
\caption{\algochain}
\begin{algorithmic}\label{algochain}
\STATE \underline{\textbf{Given}} $(\S,\succ)$ a poset with $K$ elements, $\delta>0$, $\dal(.,.)$ a dueling algorithm with input a totally ordered set and a confidence value. 

\STATE \underline{\textbf{Initialisation}} Set $\Sh=\S$, $\paretoh=\emptyset$.
\WHILE{$\Sh \neq \emptyset$}
	\STATE \textbf{Extract a maximal chain from $\Sh$:} 
		\STATE Choose $p \in \Sh$ at random, initialize $\C=\{p\}$
		\STATE $\forall q \in \Sh$, if $\C \cup \{q\}$ is a chain, set $\C \leftarrow \C \cup \{q\}$
		\STATE $\Sh \leftarrow \Sh \setminus \C$
	\STATE \textbf{Compute the maximal element of  $\C$:} 
		\STATE Obtain  $\hat{p}=\dal(\C,\delta/K)$, update $\paretoh \leftarrow \paretoh \cup \{\hat{p}\}$
	\STATE \textbf{Prune $\Sh:$}
		\STATE $\forall q \in \Sh$, if $ \hat{p}$ and $q$ are comparable, update $\Sh \leftarrow \Sh \setminus \{q\}$
\ENDWHILE

\textbf{RETURN} $ \paretoh  $
\end{algorithmic}
\end{algorithm}
Here we address the fully observable setting for the dueling bandits.

\paragraph{Fully observable posets.}
	A $K$-armed Dueling bandit on a fully observable poset $\S=\{1,\ldots,K\}$ is a dueling bandit problem such that  if $i \parallel j$, and the agent pulls the pair $(i,j)$, then the information of non-comparability is returned. This property is referred as {\em Full Observability}.

An efficient way to address this setting is to reconstruct the maximal chains of $\S$ by using the full observability property. Since every chain defines a total order, it is possible to use any total order Dueling Bandit algorithm on each of them. By carefully pruning the chain to avoid unnecessary comparisons, it is possible to efficiently recover the Pareto front of $\S$, with performances nearly as good as in the Totally order setting. This approach is detailed and analysed below.



Here, the agent may access the comparability information about any pair, 
and can thus retrieve the chains of $\S$.
The following lemma states a simple property of maximal chains, that is essential to \algochain.
\begin{lemma}\label{lem:maximalchain}
Every maximal chain $\C$ of a poset $\S\neq \emptyset$  contains a unique maximal element of $\S$.
\end{lemma}
\begin{proof}
The result follows from the transitivity property of the poset. The complete proof can be found in the proof section of the supplementary material.
\end{proof}
\begin{remark}\label{remark on chains}
By definition, it is easy to see that, conversely, 
for every maximal element $p$, there exists a (non-necessarily unique)
 maximal chain $\C$ of $\S$ such that $p \in \C$.
\end{remark}

To explore a chain in \algochain the agent has to use a dueling bandit 
algorithm $\dal$ devised for totally ordered set as a building block. 
We denote by $\dal(\C,\delta)$ the maximal element of a totally ordered set $\C$ 
returned by $\dal$ applied on the set $\C$ with confidence parameter $\delta$. 

Given $\dal$, the agent proceeds as follows. She initializes $\Sh= \S$--- 
$\Sh$ contains the elements that have not been processed yet---and $\paretoh=\emptyset$, the candidates for the Pareto front,
  and, up until $\Sh$ is empty, the agent successively  a) extract a maximal chain of $\Sh$,
 b) computes the maximal element of $\C$ (a totally ordered subset) by using $\dal$, and c) prune $\Sh$, 
 i.e. eliminates all the elements of $\Sh$ which are comparable to $\hat{p}$.
 
The upper bound on the number of pulls for \algochain to provide the Pareto front
is given by the next theorem.
\begin{theorem}\label{thm:algochain}
Assume that $\dal(\C,\delta')$ correctly returns the maximal element of $\C$ with probability at least $1- \delta'$ 
using at most $\mathcal{T}(\dal(\C,\delta'))$ pulls. Then \algochain returns the Pareto front of $\S$ with probability 
at least $1-\delta$ with at most $\mathcal{T}$ comparisons, where
$$\mathcal{T} \le \O \left( K^2 + \sum_{c\in \pareto} \max_{\C \text{ a chain containing }c} \mathcal{T}(\dal(\C,\delta \frac{\vert C \vert}{K})) \right).$$
\end{theorem}

\begin{proof}
The proof is divided into two parts: first, we only consider the event $\mathbf{E}_1$ 
where during the execution of Algorithm~\ref{algochain}, each call to  $\dal (\C, \delta/\S)$ returns 
the correct answer (the maximal element of $\C$), and we prove that on $\mathbf{E}_1$, Theorem~\ref{thm:algochain} is correct.
 Second, using a bound on the number of calls to $\dal$ performed on $\mathbf{E}_1$, we prove 
 that $\proba(\mathbf{E}_1) \ge 1- \delta$. 

The following invariant holds on $\mathbf{E}_1$\\
\textbf{Invariant:} at the beginning of each iteration of the while loop, we have 
\begin{equation}\label{eq : invariant for chain}
\paretoh \subset \pareto \subset \paretoh \cup \Sh \text{ and } \forall p \in \paretoh, \forall q \in \Sh, 
 p \parallel q 
\end{equation}
A consequence is that $\paretoh$ increases by one element at 
each iteration of the while loop, and thus $\dal$ is called exactly $\vert \pareto \vert$ times, 
after which \eqref{eq : invariant for chain} implies $\paretoh = \pareto$, hence
\begin{align*}
\proba(\mathbf{E}_1^C) & \le \vert \pareto \vert \delta/\vert S \vert \le \delta.
\end{align*}

The number of additional comparisons required to build the chain is upper bounded by $K^2$,
 as all pairs of arms have to be compared at most once. Hence, the upper bound on $\mathcal{T}$ 
 is derived from the fact that at each iteration, a chain with a different element of $c \in \pareto$ 
 is considered. All the details of the proof can be found in the devoted section of the supplementary material.
\end{proof}

The following corollary illustrates Theorem~\ref{thm:algochain} when 
$\dal$ is the Interleaved Filter algorithm \cite{yue2012if}.
\begin{corollary}\label{cor:algochain and if2}
Assume that $(\S,\succe)$ satisfies the strong stochastic transitivity and the triangle inequality of \citet{yue2012if}.
Then \algochain using the IF2 algorithm as $\dal$ will return the correct Pareto front $\pareto$ with probability at least $1-\delta$ in at most $T$ steps, where 
$$ T \le \O \left(K^2  + \frac{K}{(d(\pareto) )^2} {\log(K^2/\delta)}\right). $$

\end{corollary}

Interestingly, when $\S$ is totally ordered, there is {\em one} maximal chain, $\S$, and \algochain reduces to $\dal$.


\section{Additional Numerical Simulations}
\label{sec:experiments appendix}

\begin{figure*}[tb]\label{fig:results appendix}
\begin{center}
\includegraphics[scale=0.2]{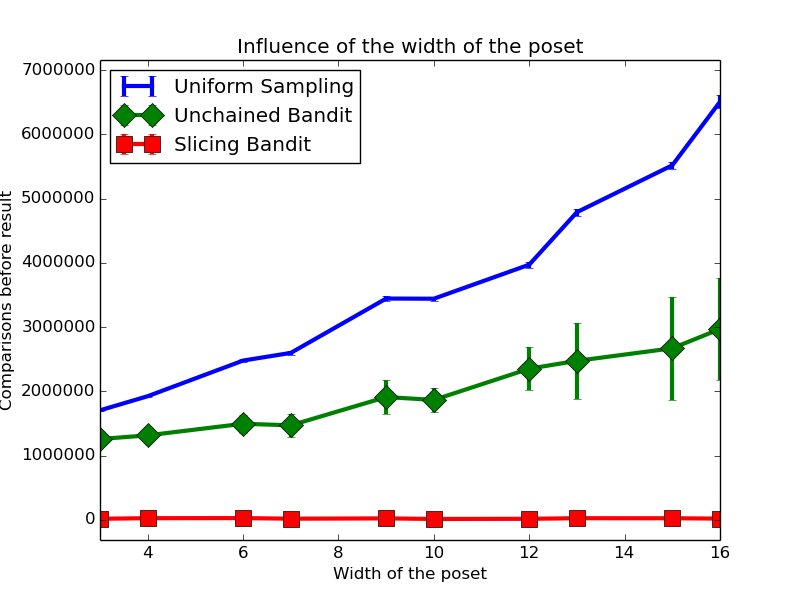}
\includegraphics[scale=0.2]{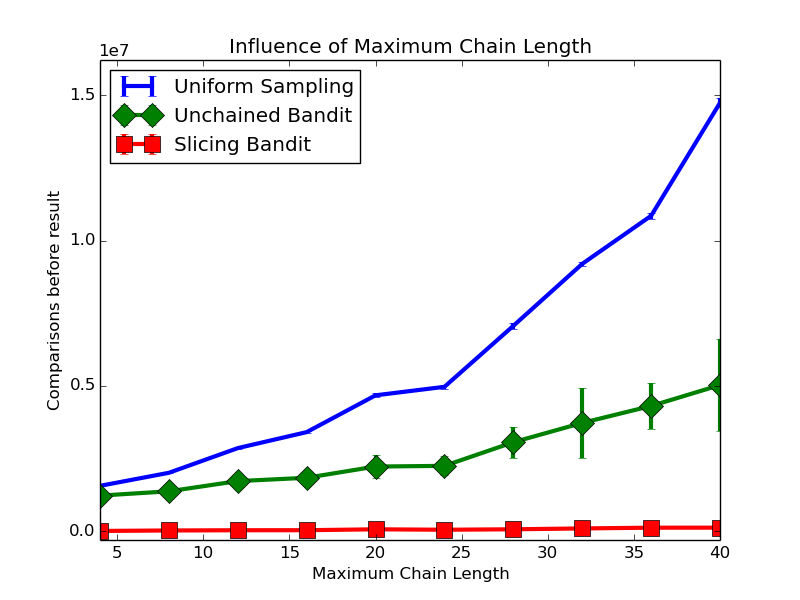} 
\includegraphics[scale=0.2]{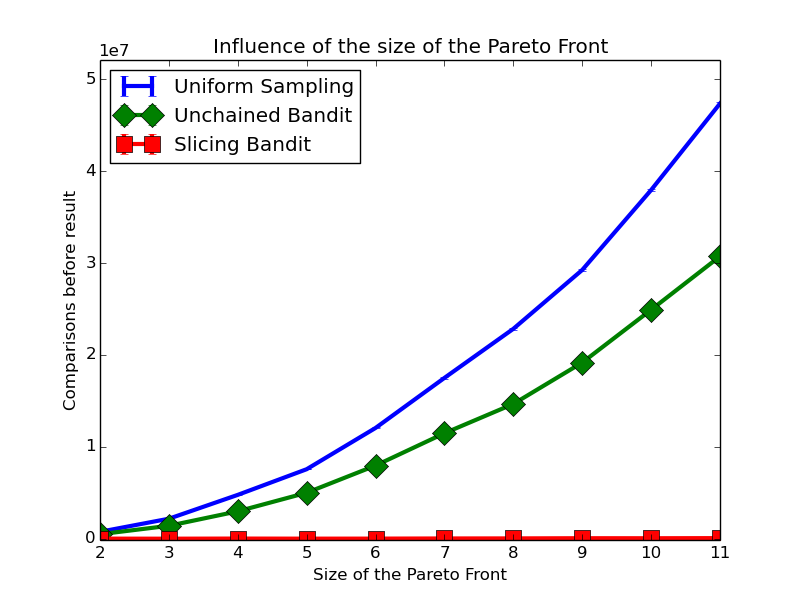}
\end{center}
\caption{Time necessary to reach a conclusion for \algochain and  \algopeeling compared to {\small \sf UniformSampling}, when the structure of the poset varies. Dependence on (left:) width of the poset, (center:) height of the poset and (right:) size of the Pareto front.}
\end{figure*}

The following experiments evaluate the relative efficiency of \algochain and \algopeeling, we confront them with randomly generated posets,
 with different sizes, widths and heights.

Given the size $p>0$ of the Pareto front, the desired width $w\ge p$ and the height $h>0$, the posets are generated as follows:
 first, a Pareto front of size $p$ is created. Then $w$ chains of length $h-1$ with no common elements are added. Finally, the top of the chains are connected to a random number of elements of the Pareto front. This creates the structure of the poset (i.e. the partial order $\succ$).  Finally, the exact values of the $\gamma_{ij}$'s are obtained from a uniform distribution, conditioned to satisfy the partially (or fully) observable framework.
When needed, $\Delta$-decoys are created according to Proposition~\ref{prop: add decoy}.

For each experiment reported on Figure~\ref{fig:results appendix}, we changed the value of one parameter, and left the other to 
their default values ($p=5$, $w=2p$, $h=10$). The results are averaged over ten runs. By default, we use $\delta= 1/1000$ and
 $\Delta=1/100$. The $(\ep_t)_t$ are generated following the procedure presented in Section~\ref{sec:Contribution} with $\Delta_0=0.25$.

We did not compare our algorithms to dueling bandits algorithms from the literature,
 as a) they fail to consider the incomparability information 
  and b) they are generally designed to return only one \emph{best} element. Instead, we use a baseline algorithm,
   {\small \sf UniformSampling} inspired from the successive elimination algorithm \cite{even2006successiveelimination}, which
    simultaneously compares  all possible pairs of arms  until one of the arms appears suboptimal, at which point
	 it is removed from the set of selected arms. When only $\Delta$-indistinguishable elements remain, it uses $\Delta$-decoys. 

We note that \algochain clearly outperforms the other algorithms by a wide margin, thanks to the access to the comparability information and the careful management of chains.
For partially observable posets, \algopeeling produces much better results than {\small \sf UniformSampling} and its advantage increases with the complexity of the problem.

\section{Appendix : Extended Proofs}

\subsection*{Proof of Lemma 3.1}
\textbf{Existence:} Since $\C$ is a finite totally ordered set, it admits an unique maximal element. Let $c \in \C$ be the maximal element of $\C$. We use reductio ad absurdum. Suppose that $c$ is not a maximal element of $\S$. By definition of maximal element, $\exists c' \in \S $ such that $c' \succ c$. But $\forall c'' \in \C$, we have $c \succe c''$, then by transitivity $c'\succ c''$. Hence $\C \vee \{c' \}$ is a chain which strictly contains $\C$, which contradicts the fact the $\C$ is a maximal chain.

\textbf{Uniqueness:} let $c,c' \in \C$ be two maximal element of $\S$. Since $\C$ is a chain, $c$ and $c'$ are comparable. Since $c$ is a maximal element, we have $c \succe c'$. The same is true for $c'$, hence the conclusion.


\subsection*{Proof of Theorem 1}

Let $\mathbf{E}_1$ be the event where during the execution of Algorithm 1, each call to  $\dal (\C, \delta/K)$ return the correct answer (the maximal element of $\C$).

The proof is divided into two steps : First, we are going to prove that on $\mathbf{E}_1$, Theorem 1 is correct. Then, using an upper bound of the number of call to $\dal$ done on the event $\mathbf{E}_1$, we will prove that $\proba(\mathbf{E}_1^C) \le \delta$, hence the conclusion.

On $\mathbf{E}_1$,consider the following invariant :\\
\textbf{Invariant :} At the beginning of each iteration of the while loop, we have 
\begin{equation}\label{eqp : increase pareto}
\paretoh \subset \pareto
\end{equation}
\begin{equation}\label{eqp: paretoh is uncomparable to sh}
\forall p \in \paretoh,\quad \forall q \in \Sh,\quad 
 p \parallel q
\end{equation}
\begin{equation}\label{eqp: sh contains the remaining of pareto}
\pareto \subset \paretoh \cup \Sh
\end{equation}

It is easy to see that the invariant is true at the beginning of the algorithm, because at the initialisation, $\paretoh= \emptyset$  and $\Sh=\S.$

Assume that the invariant is true at the beginning of the loop $t+1$, and denote by $\Sh_t$,$\paretoh_t$ the value of $\Sh$,  $\paretoh$ at the end of loop $t$.

Since the algorithm has not stopped, $\Sh_t$ is not empty.
 By definition, the subset $\C$ constructed by the algorithm is a maximal chain of $\Sh$. Since $\C$ is a non empty finite totally ordered set, it admits an unique maximum element $c$.

We prove that 
\begin{equation}\label{eqp: c in pareto}
c\in \pareto
\end{equation}
with reductio ad absurdum (RAA for short). Assume that $c \notin \pareto$. Then $\exists c' \in \pareto$ such that $c' \succ c.$ Since $\C$ is a maximal chain of $\Sh$, it implies that $c' \notin \Sh$. Hence \eqref{eqp: sh contains the remaining of pareto} implies that $c' \in \paretoh$. But then $c' \succ c$ contradicts \eqref{eqp: paretoh is uncomparable to sh}, which concludes the RAA.

Note that \eqref{eqp: c in pareto} and \eqref{eqp: paretoh is uncomparable to sh} implies
\begin{equation}\label{eqp: c in pareto_new}
c\in \pareto\setminus\paretoh.
\end{equation}

Then, on $\mathbf{E}_1$, $\dal(\C,\delta/K)=c$, and 
\begin{equation}\label{eqp : increase paretois still true}
\paretoh_t \varsubsetneq \paretoh_t \cup \{ c\} = \paretoh_{t+1} \subset \pareto.
\end{equation}
Now by construction we have
\begin{align*}
\Sh_{t+1}&=\{p \in \Sh_t,\quad p \succ c \text{ or } p \parallel c \} \\
&=\{p \in \Sh_t,\quad p \parallel c \}
\end{align*}
since $c \in \pareto.$ Then \eqref{eqp: paretoh is uncomparable to sh} implies that 
\begin{equation}\label{eqp: paretoh is uncomparable to sh still true}
\forall p \in \paretoh_{t+1},\quad \forall q \in \Sh_{t+1},\quad 
 p \parallel q.
\end{equation}
Finally, we prove with RAA that
\begin{equation}\label{eqp: sh contains the remaining of pareto still true}
\pareto \subset \paretoh_{t+1} \cup \Sh_{t+1}
\end{equation}
Let $p\in \pareto$ such that $p\notin \paretoh_{t+1} \cup \Sh_{t+1}$. \eqref{eqp: sh contains the remaining of pareto} implies that $p\in \paretoh_{t} \cup \Sh_{t}$. Since $\paretoh_{t+1} \supset \paretoh_{t}$, we have $p \in \Sh_{t+1} \setminus \Sh_{t}$. Then, by definition of $\Sh_{t+1}$, we have $c \succ p$, which contradicts $p \in \pareto$ and conclude the RAA.

Finally, \eqref{eqp : increase paretois still true}\eqref{eqp: paretoh is uncomparable to sh still true} and \eqref{eqp: sh contains the remaining of pareto still true} implies that the invariant is true at the beginning of the loop $t+2$.

When the algorithm stops, we have $\Sh= \emptyset$, hence \eqref{eqp : increase pareto} and \eqref{eqp: sh contains the remaining of pareto} implies that 
$$ \paretoh \subset \pareto \subset \paretoh \cup \emptyset = \paretoh$$
that is to say $\paretoh=\pareto$. Hence on $\mathbf{E}_1$, Algorithm \ref{algochain} reaches the correct conclusion.

A consequence of \eqref{eqp: c in pareto_new} is that $\paretoh_t$ increases by exactly one element at each iteration of of the while loop, and thus the $\dal$ is called exactly $\vert \pareto \vert$ times. Hence, if we denote by $\C_t$ the chain constructed at the loop $t$, 
\begin{align*}
\proba(\mathbf{E}_1^C) & \le \sum_{ t=1 } ^{ \vert \pareto \vert } \proba\left(\{ \dal(\C_t, \delta /K )\text{ failed } \}\right)\\
& \le  \sum_{ c \in \pareto } \delta / K \le \delta,
\end{align*}

Additionally, the number of additional comparisons required to build all the chains is upper bounded by $K^2$, as all pair of elements have to be compared at most once. Hence, the upper bound of $\mathcal{T}$ is derived from the fact due to \eqref{eqp: c in pareto_new}, at each iteration, a chain with a different element of $c \in \pareto$ is considered.


\subsection*{Proof of Corollary 3.1}

Let $\C_t$ be the chain considered by Algorithm 1 during the loop t, and we denote by $c_t$ the maximal element of $\C_t$, which is  the unique element of $\pareto \cap \C_t$ (consequence of \eqref{eqp: c in pareto_new}).
Theorem 2 from \\ \citep{yue2012if} implies that in this case,
\begin{align*}
\mathcal{T} \left( \dal  \left( \C_t, \delta / K \right) \right) &\le \O \left( \vert C_t  \vert \frac{\log( \vert C_t \vert ^2 K /\delta )}{(\min_{c'\in \C_t } \gamma_{c_t c'} )^2 }  \right)\\
&\le \O \left( \vert C_t  \vert \frac{\log(  K^3 /\delta )}{(\min_{c \in \pareto, c'\in \S, c\succ c'} \gamma_{cc'} )^2  } \right).
\end{align*}
Using that by construction, $\forall t<t', \C_t \cap \C_t' =\emptyset$, and $\bigcup_{t} \C_t = \S$, we have
\begin{align*}
\sum_{t=1}^{\vert \pareto \vert} \mathcal{T} \left( \dal  \left( \C_t, \delta / K \right) \right)
&\le \O \left( \sum_{t=1}^{\vert \pareto \vert} \vert C_t  \vert \frac{\log(  K^3 /\delta )}{(\min_{c \in \pareto, c'\in \S, c\succ c'} \gamma_{cc'} )^2  } \right)\\
&\le \O \left( K \frac{\log(  K^3 /\delta )}{(\min_{c \in \pareto, c'\in \S, c\succ c'} \gamma_{cc'} )^2  } \right)
\end{align*}
Hence the conclusion.


\subsection*{Proof of Proposition 3.7}
\underline{\textbf{Case $\dal$ = Algorithm 3}}\\
In this setting, the arms are compared using decoys.

We are going to proceed as in the proof of Theorem 1.

Let $\mathbf{E}_1$ be the event where during the execution of Algorithm 5, each call to Algorithm 3 returns the correct answer. We are going to prove the following invariant for the principal loop of the Algorithm on $\mathbf{E}_1$.

\textbf{Invariant:}  At the iteration $n$, Let $\S_t^n$ the set of element of $\S_t$ already considered, $\paretoh^n$ the current set of pivot. Then 
\begin{align}
\forall c' \in \S_t^n \quad \exists c \in \paretoh^n,  \quad & c \succe c'\label{eqp2: ph is pseudo pareto} \\
\forall c,c' \in \paretoh_t, \quad & c \parallel c' \label{eqp2: ph is antichain}
\end{align}

It is easy to see that the invariant is true at the beginning of the algorithm because 
$\S_t^0= \paretoh^0$ and $\vert \paretoh^0 \vert = 1.$

Suppose that the invariant is true at the $n$-th iteration. Let $p$ be the new element considered, i.e. $\S^{n+1}_{t} = \S^n_t \cup \{ p \}$, and define $\Gamma_-^p\doteq \{ q \in \paretoh^n, p \succ q \}$
\begin{enumerate}
\item \textbf{Case 1.} $\exists q \in \paretoh^n$ s.t. $q\succ p$. In this case, $\paretoh^{n+1}=\paretoh^n \setminus \Gamma_-^p$, hence \eqref{eqp2: ph is antichain} at iteration $n$ immediatly implies \eqref{eqp2: ph is antichain} at iteration $n+1.$ Since $q \succ p$, we have $\forall q' \in \Gamma_-^p$, we have $q \succ q'$ by transitivity.Hence \eqref{eqp2: ph is pseudo pareto} at iteration n    implies \eqref{eqp2: ph is pseudo pareto} at iteration n+1.
\item \textbf{Case 2.}  $\forall q \in \paretoh^n$, $p\succ q$ or $p \parallel q$.  Then $$\paretoh^{n+1}=\{p\} \cup \paretoh^n \setminus \Gamma_-^p,$$ 
and is it easy to see that \eqref{eqp2: ph is pseudo pareto} is still true iteration $n+1$. Now we are going to prove that \eqref{eqp2: ph is antichain} is still true by RAA. Assume that  $\exists q \in \paretoh^{n+1}$ s.t. $q$ is comparable to $p$. By definition of $\Gamma_-^p$, it implies that $q \succ p$, which contradicts the initial assumption of the case.
\end{enumerate}

After the last iteration n, we have $\S^{n+1}_t = \S_t$, since all the elements have been examined.
We now prove by RAA that the invariant implies that $\paretoh^{n+1}=\pareto.$  We drop the $n+1$ in $\paretoh^{n+1}$ for the sake of alleviating the notations.

Suppose that $\paretoh  \not\subset \pareto$ and let $p \in \paretoh \setminus \pareto$. Since $p \notin \pareto, \exists q \in \pareto $ s.t. $ q \succ p$. If $q \in \paretoh$, \eqref{eqp2: ph is antichain} is contracted. Then $q \notin \paretoh.$ Hence $ q \succ p$ contradicts \eqref{eqp2: ph is pseudo pareto}. So $\paretoh \subset \pareto$.

Now assume that $\pareto  \not\subset \paretoh$ and let $p \in \pareto \setminus \paretoh$. Since $ p \notin \paretoh,$ \eqref{eqp2: ph is pseudo pareto} implies that $\exists q \in \paretoh$ s.t. $q \succe p$. Since $p \notin \paretoh$ and $q\in \paretoh$, $q\neq p$ hence $q \succ p$, which contradicts $p\in \pareto$. So $\pareto \subset \paretoh.$ Hence $\paretoh = \pareto.$

A consequence of \eqref{eqp2: ph is antichain} is that at each step, $\paretoh^n$ is an antichain. Since during the execution of the algorithm all the elements of $S_t$ are compared to all the element of the current $\paretoh$, the algorithm do at most 

$$ \vert \S_t \vert \max_n \vert \paretoh^n \vert \le \vert \S_t \vert \text{width}(\S_t)$$ comparisons, and as a consequence
\begin{align*}
\proba( \mathbf{E}_1^C ) \le  \vert \S_t \vert \text{width}(\S_t) \frac{\delta}{\vert \S_t \vert^2} \le \delta.
\end{align*}

The upper bound on the number of comparisons results with the same remark combined with Proposition 3.5.

\underline{\textbf{Case $\dal$ = Algorithm 2}.}\\
During the epochs $t<N$, the arms are compared directly to each other, i.e. Algorithm 2 is used for comparisons purpose.
We first tackle the case $t=1$, i.e. the first epoch, since in this case, there is no previous observations, and thus no negative term in the upper bound.\\
\underline{\textbf{Case $t=1$}}. The proof for $t=1$ unfolds similarly to the previous case, with a different invariant.

Let $\mathbf{E}_1$ be the event where during the execution of Algorithm 5, each call to Algorithm 2 returns the correct answer e.g. $i \succ j$ (resp $j \succ i$) if and only if $\gamma_{ij}> \ep$ (resp $\gamma_{ji}> \ep$). We are going to prove the following invariant for the principal loop of the Algorithm on $\mathbf{E}_1$.

\textbf{Invariant:}  At the iteration $n$, Let $\S_t^n$ the subset of element of $\S_t$ already considered, $\paretoh^n$ the current set of pivot. Then 
\begin{align}
\forall c' \in \S_t^n \quad \exists c \in \paretoh^n,  \quad & c \succe c'\label{eqp3: ph is pseudo pareto} \\
\forall c,c' \in \paretoh_t, \quad & c \pareps c'  \label{eqp3: ph is ep antichain}
\end{align}

It is easy to see that the invariant is true at the beginning of the algorithm because 
$\S_t^0= \paretoh^0$ and $\vert \paretoh^0 \vert = 1.$

Suppose that the invariant is true at the $n$-th iteration. Let $p$ be the new element considered, i.e. $\S^{n+1}_{t} = \S^n_t \cup \{ p \}.$ 
\begin{enumerate}
\item \textbf{Case 1.} $\exists q \in \paretoh^n$ s.t. $q\succ p$ and $\gamma_{qp}>\ep$. In this case, $\paretoh^{n+1}=\paretoh^n \setminus \Gamma_-^p$, hence \eqref{eqp3: ph is  ep antichain} at iteration $n$ immediatly implies \eqref{eqp3: ph is ep antichain} at iteration $n+1.$ Since $q \succ p$, we have $\forall q' \in \Gamma_-^p$, we have $q \succ q'$ by transitivity. Hence \eqref{eqp3: ph is pseudo pareto} at iteration $n$    implies \eqref{eqp3: ph is pseudo pareto} at iteration $n+1$.
\item \textbf{Case 2.}  $\forall q \in \paretoh^n$, ($p\succ q$ and $\gamma_{pq}>\ep$ ) or $\vert \gamma_{pq} \vert<\ep$.  Then $$\paretoh^{n+1}=\{p\} \cup \paretoh^n \setminus \Gamma_-^p,$$ 
and is it easy to see that \eqref{eqp2: ph is pseudo pareto} is still true iteration $n+1$. Now we are going to prove that \eqref{eqp2: ph is antichain} is still true by RAA. Assume that  $\exists q \in \paretoh^{n+1}$ s.t. $q$ is comparable to $p$ and $\vert \gamma_{qp}\vert >\ep$. By definition of $\Gamma_-^p$, it implies that $q \succ p$, and the order compatibility of the poset implies that $\gamma_{qp}>\ep$ which contradicts the initial assumption of the case.
\end{enumerate}

After the last iteration n, we have $\S^{n+1}_t = \S_t$, since all the elements have been examined.
We now prove by RAA that the invariant implies that $\paretoh^{n+1}$ is an $\ep$-approximation of $\pareto$.  We drop the $n+1$ in $\paretoh^{n+1}$ for the sake of alleviating the notations.

Now assume that $\pareto  \not\subset \paretoh$ and let $p \in \pareto \setminus \paretoh$. Since $ p \notin \paretoh,$ \eqref{eqp2: ph is pseudo pareto} implies that $\exists q \in \paretoh$ s.t. $q \succe p$. Since $p \notin \paretoh$ and $q\in \paretoh$, $q\neq p$ hence $q \succ p$, which contradicts $p\in \pareto$. So $\pareto \subset \paretoh.$

Now suppose that $\exists q \in \paretoh$  such that $\exists p \in \pareto $ s.t. $p \succ q$ and  $\gamma_{pq} > \ep.$ Since $\pareto \subset \paretoh$, we have $p \in \paretoh$ and thus $\gamma_{pq} > \ep.$ contradicts \eqref{eqp3: ph is ep antichain}. Hence $\paretoh$ is a $\ep$-approximation of $\pareto.$

A consequence of \eqref{eqp3: ph is ep antichain} is that at each step, $\paretoh^n$ is an $\ep$-antichain. Since during the execution of the algorithm all the elements of $S_t$ are compared to all the element of the current $\paretoh$, the algorithm do at most 

$$ \vert \S_t \vert \max_n \vert \paretoh^n \vert \le \vert \S_t \vert \width_\ep(\S_t)$$ comparisons, and as a consequence
\begin{align*}
\proba( \mathbf{E}_1^C ) \le  \vert \S_t \vert \width_\ep(\S_t) \frac{\delta}{\vert \S_t \vert^2} \le \delta.
\end{align*}

The upper bound on the number of comparisons results with the same remark combined with the fact that Algorithm 2 uses Hoeffding inequality.

\underline{\textbf{Case $1<t<N$.}}\\
To conclude, we only need to lower bound the number of previous comparisons that can be reused.
Once again, consider the event $\mathbf{E}_1$ be the event where during the execution of Algorithm 5, each call to Algorithm 2 returns the correct answer e.g. $i \succ j$ (resp $j \succ i$) if and only if $\gamma_{ij}> \ep$ (resp $\gamma_{ji}> \ep$).
Let $i$ and $j \in \S_t$ such that  $i$ and $j$ are compared at epoch $t$ (i.e. during the call number $t$ of Algorithm \ref{algopeeling}).
Note that $S_t=\paretoh^n_{t-1}$ and let assume without any loss of generality that $i$ was added before $j$ into $\paretoh^n_{t-1}$. Since $i$ is a pivot at the end of the epoch $t-1$, it was compared to all the arm considered after $i$, including $j$.

Since both $i$ and $j$ are pivots at the end of epoch $t-1$, it implies that 
$i \parallel j$ or $\gamma_{ij}<\ep_{t-1}$. In both cases, Algorithm the algorithm does exactly $\frac{\log(K^2 /\delta')}{\ep_{t-1}^2}$ comparisons to reach this conclusion. The result follows from the reuse of information.


\subsection*{Proof of Theorem 2.}

First note that if $\pareto'$ is a $\ep$- approximation of $\pareto$, then $\pareto \subset \pareto'$. Additionally, it is easy to see that if $\S$ is a poset and $\pareto$ is its Pareto set, then $\forall \S'\subset \S$ such that $\pareto \subset \S'$, the Pareto front of $\S'$ is $\pareto.$

Hence, Proposition 3.7  implies that with probability at least $1-N\delta/N = 1- \delta$, Algorithm 4 returns the pareto front of $\S$.
 in at most T comparisons, where 
\begin{equation*}
\begin{aligned}
T &\le  2\sum_{t=1}^{N-1} \vert \S_t \vert \width_{\ep_t}(\S_t) {\log(2N\vert \S_t \vert^2 /\delta) }\left(\frac{1}{\ep_t^2}- \mathbf{1}_{t>1}\frac{1}{\ep_{t-1}^2} \right) + 4\vert \S_N \vert \text{\textbf{width}}(\S_N) \frac{\log(4N\vert \S_N \vert^2/\delta) }{\Delta^2} \\
&\le  2\sum_{t=1}^{N-2}\frac{1}{\ep_t^2} \left( \vert \S_t \vert \width_{\ep_t}(\S_t) {\log(2N\vert \S_t \vert^2 /\delta) } - \vert \S_{t+1} \vert \width_{\ep_{t+1}}(\S_{t+1}) {\log(2N\vert \S_{t+1} \vert^2 /\delta) }  \right) \\
& \quad + \frac{2}{\ep_{N-1}^2} \vert \S_{N-1} \vert \width_{\ep_{N-1}}(\S_{N-1}) {\log(2N\vert \S_{N-1} \vert^2 /\delta) } + 4\vert \S_N \vert \text{\textbf{width}}(\S_N) \frac{\log(4N\vert \S_N \vert^2/\delta) }{\Delta^2} 
\end{aligned}
\end{equation*}

where the second inequality is obtained by rearranging the sum.
Now,  by hypothesis we have 
$$\ep_t > \ep_{N-1} \ge \Delta \sqrt{\frac{\vert \S \vert}{\width(\S)}}$$
Hence, since the $\vert \S_t \vert \width_{\ep_t}(\S_t) {\log(N\vert \S_t \vert^2 /\delta) }$ is decreasing in $t$ we have

\begin{equation*}
\begin{aligned}
T &\le  2\sum_{t=1}^{N-2}\frac{\width(\S)}{\vert \S \vert \Delta^2} \left( \vert \S_t \vert \width_{\ep_t}(\S_t) {\log(2N\vert \S_t \vert^2 /\delta) }  - \vert \S_{t+1} \vert \width_{\ep_{t+1}}(\S_{t+1}) {\log(2N\vert \S_{t+1} \vert^2 /\delta) }  \right) \\
& \quad + \frac{2 \width(\S)}{\vert \S \vert \Delta^2} \vert \S_{N-1} \vert \width_{\ep_{N-1}}(\S_{N-1}) {\log(2N\vert \S_{N-1} \vert^2 /\delta) }+ 4\vert \S_N \vert \text{\textbf{width}}(\S_N) \frac{\log(N\vert \S_N \vert^2/\delta) }{\Delta^2} \\
&\le  \frac{2}{\Delta^2} \vert \S \vert \frac{\width_{\ep_{1}}(\S)}{\vert \S \vert } \width(S){\log(2N\vert \S \vert^2 /\delta) } + 4\vert \S_N \vert \text{\textbf{width}}(\S_N) \frac{\log(4N\vert \S_N \vert^2/\delta) }{\Delta^2} \\
&\le  \O\left( K  \width(\S) \frac{\log(N K^2/\delta) }{\Delta^2}\right) 
\end{aligned}
\end{equation*}

\subsection*{Proof of Theorem 3.}

We know from Proposition \ref{prop:unchainedregretfinal}, with probability at least $1- \delta$, the algorithm does not reach an incorrect result the comparison. For the rest of the proof, we restrict ourselves to this event.

First we consider the regret $\reg_0$ induced by the peeling process.

Let $i$ be an arm, and $N_i$ be the last peeling step before $i_p$ is eliminated. If $i$ is not eliminated at the end of the peeling, then we set $N_i=N-1$.  In other words,

\begin{align*}
 N_i & = \max \{ 1 \le t \le N-1, \quad i \in \paretoh_t \}  \\
 & = \min \left( \lceil \frac{\log(\Delta_i)}{\log(\gamma)} \rceil, N-1 \right).
 \end{align*}
 
Let $j\le N_i$. During the $j$-th phase of peeling, the arm $i$ is compared to at most $\vert S_j -1 \vert$ other arms. Hence, with the same argument as in Proposition \ref{prop:unchainedregretfinal}, we have
 
 $$ \reg_0 \le 2 \sum_{i=1}^K \Delta_i  \sum_{t=1}^{N_i} \vert \S_t \vert {\log(2N\vert \S_t \vert^2 /\delta) }\left(\frac{1}{\ep_t^2}- \mathbf{1}_{t>1}\frac{1}{\ep_{t-1}^2} \right). $$
 
 Now since $\ep_t < \ep_{t-1}$, and  by hypothesis, $\vert S_t \vert \le \alpha^{t-1} K$, we have
 
 \begin{align*}
\reg_0 & \le 2K \log( \frac{2N K^2}{\delta}) \sum_{i=1}^K \Delta_i  \sum_{t=1}^{N_i} \alpha^{t-1} \left(\frac{1}{\ep_t^2}- \mathbf{1}_{t>1}\frac{1}{\ep_{t-1}^2} \right) \\
&\le 2K \log( \frac{2N K^2}{\delta})\sum_{i=1}^K \Delta_i  \left( \sum_{t=1}^{N_i-1} \frac{\alpha ^{t-1}}{\ep_t^2} \left( 1 - \alpha\right) + \frac{\alpha ^{N_i-1}}{\ep_{N_i}^2} \right).
 \end{align*}
 
 Since by construction, we have $\ep_{t+1} = \gamma \ep_t$, then
 
  \begin{align*}
\reg_0 & \le 2K \log( \frac{2N K^2}{\delta})\sum_{i=1}^K \frac{\Delta_i}{\ep_{N_i}^2}  \left( \sum_{t=1}^{N_i-1} \gamma^{2 (N_i -t)} \alpha ^{t-1}\left( 1 - \alpha\right) + \alpha ^{N_i-1}\right)\\
& \le 2K \log( \frac{2N K^2}{\delta})\sum_{i=1}^K \frac{\Delta_i}{\ep_{N_i}^2}  \left( \gamma^{2 (N_i-1)} \sum_{t=1}^{N_i-1} (\frac{\alpha }{\gamma^2 })^{t-1}  \left( 1 - \alpha\right) + \alpha ^{N_i-1}\right)\\
& \le \frac{2K}{\gamma^2} \log( \frac{2N K^2}{\delta})\sum_{i=1}^K \frac{1}{\Delta_i}  \left( \gamma^{2 (N_i-1)} \sum_{t=1}^{N_i-1} (\frac{\alpha }{\gamma^2 })^{t-1}  \left( 1 - \alpha\right) + \alpha ^{N_i-1}\right),
 \end{align*}
 since by definition of $N_i$, we have $\Delta_i < \ep_{N_i-1} = \gamma \ep_{N_i}$. Now we have to consider two cases.
 
\noindent  \textbf{Case $\gamma^2 \neq  \alpha$:}

  \begin{align*}
\reg_0 & \le \frac{2K}{\gamma^2} \log( \frac{2N K^2}{\delta})\sum_{i=1}^K \frac{1}{\Delta_i}    \left( \gamma^{2 (N_i-1)}    \left( 1 - \alpha\right) \frac{1-(\alpha/\gamma^2 )^{N_i-1}}{1- (\alpha / \gamma^2 )} + \alpha ^{N_i-1}\right)\\
& \le \frac{2K}{\gamma^2} \log( \frac{2N K^2}{\delta})\sum_{i=1}^K \frac{1}{\Delta_i}   \frac{\gamma^{2N_i}(1-\alpha) + \alpha^{N_i}(\gamma^2 -1)  }{\gamma^2 - \alpha}\\
& \le \frac{2K}{\gamma^2} \log( \frac{2N K^2}{\delta})\sum_{i=1}^K \frac{1}{\Delta_i}   C_{\alpha,\gamma}(N_i),
 \end{align*}
and

\noindent  \textbf{Case $\gamma^2 =  \alpha$:}

  \begin{align*}
\reg_0 & \le \frac{2K}{\gamma^2} \log( \frac{2N K^2}{\delta})\sum_{i=1}^K \frac{1}{\Delta_i}   \left( \alpha^{N_i-1}    \left( 1 - \alpha\right) (N_i -1) + \alpha ^{N_i-1}\right)\\
& \le \frac{2K}{\gamma^2} \log( \frac{2N K^2}{\delta})\sum_{i=1}^K \frac{1}{\Delta_i}     \alpha^{N_i-1}  N_i \\
& \le \frac{2K}{\gamma^2} \log( \frac{2N K^2}{\delta})\sum_{i=1}^K \frac{1}{\Delta_i}   C_{\alpha,\gamma}(N_i),
 \end{align*}
hence the conclusion.

Now let $\reg_1$ be the regret generated by the decoy step. To reach this step, an arm $i$ must be such that  $\Delta_i< \ep_{N-1}$. If $i \in \pareto$, then pulling the arm $i$ produces no regret. Otherwise, it is easy to see that the  arm is compared to at most $\width(S)$ other arms before being eliminated.

  \begin{align*}
\reg_1 & \le \width(S) \log( \frac{2N K^2}{\delta})\sum_{i, \Delta_i < \ep_{N-1}, i \notin \pareto}  \frac{\Delta_i}{\Delta^2}  \\ 
& \le (\frac{\ep_{N-1}}{\Delta})^2   \width(S) \log( \frac{2N K^2}{\delta})\sum_{i, \Delta_i < \ep_{N-1}, i \notin \pareto}  \frac{1}{\Delta_i}\\
& \le K  \width(S) \log( \frac{2N K^2}{\delta})\sum_{i, \Delta_i < \ep_{N-1}, i \notin \pareto}  \frac{1}{\Delta_i}.
 \end{align*}

\end{document}